\documentclass[conference]{IEEEtran}
\IEEEoverridecommandlockouts
\usepackage{amsmath,amsfonts}

\usepackage{algorithm}
\usepackage{algorithmic}
\usepackage{multirow}
\usepackage{listings}%
\usepackage{tabularx}
\usepackage{textcomp}
\usepackage{stfloats}
\usepackage{url}
\usepackage{verbatim}
\usepackage{amsthm}
\usepackage{graphicx}
\usepackage{bbm}
\usepackage{cite}
\usepackage{color}
\UseRawInputEncoding  
\usepackage{subcaption}
\usepackage{booktabs}
\usepackage{caption}
\usepackage{subcaption}
\newtheorem{assumption}{Assumption}
\newtheorem{theorem}{Theorem}
\renewcommand{\algorithmicrequire}{\textbf{Input:}}

\usepackage[T1]{fontenc}
\begin{document}
\title{Multi-Worker Selection based Distributed Swarm Learning for Edge IoT with Non-i.i.d. Data
\thanks{This work was supported in part by the National Science Foundation grants
\#2128596, \#2146497, \#2231209, \#2244219, \#2315596, \#2343619, \#2349878, \#2416872, and \#2413622.}
}
\author{\IEEEauthorblockN{Zhuoyu Yao$^\dagger$, \; Yue Wang$^\dagger$, \; Songyang Zhang$^*$, \; Yingshu Li$^\dagger$, \; Zhipeng Cai$^\dagger$, \; Zhi Tian$^\star$}

\IEEEauthorblockA{$^\dagger$Department of Computer Science, Georgia State University, Atlanta, GA, USA\\$^*$Department of Electrical and Computer Engineering, University of Louisiana at Lafayette, LA, USA\\$^\star$Department of Electrical and Computer Engineering, George Mason University, Fairfax, VA, USA}}

\maketitle

\begin{abstract}
Recent advances in distributed swarm learning (DSL) 
offer a promising paradigm 
for 
edge Internet of Things. 
Such advancements 
enhance 
data privacy, communication efficiency, energy saving, and 
model scalability. However, the presence of non-independent and identically distributed (non-i.i.d.) data pose a significant challenge for multi-access edge computing, 
degrading learning performance and diverging training behavior of vanilla DSL. 
Further, there still lacks theoretical guidance on how 
data heterogeneity affects model training accuracy, which requires 
thorough investigation. 
To fill the gap, this paper first study the data heterogeneity by measuring the impact of non-i.i.d. datasets under the DSL framework. This then motivates a new multi-worker selection design for DSL, termed M-DSL algorithm, which works effectively with distributed heterogeneous data. 
A  
new non-i.i.d. degree metric is introduced and defined in this work to formulate the statistical difference among 
local datasets, which builds a connection between the measure of data heterogeneity and the evaluation of DSL  performance. 
In this way, our M-DSL guides effective selection of multiple works who make prominent contributions for global model updates. We also provide theoretical analysis on the convergence behavior of our M-DSL, 
followed by extensive experiments on 
different heterogeneous datasets and non-i.i.d. data settings. Numerical 
results verify 
performance improvement and network intelligence enhancement provided by our M-DSL beyond the benchmarks. 

\end{abstract}

\section{Introduction}

With the tremendous technical breakthrough in 6G networks, edge computing emerges as a remarkable alternative to the cloud-based architectures for more collaborative, low-latency, and reliable solutions\cite{fan2025enhancing}. 
Distributed learning systems, e.g. via federated learning (FL), perform as a dedicated workhorse to leverage the local data and computational resources over various edge Internet of Things (IoTs) devices, 
as advantages  in distributed data utilization, collaborative learning, privacy protection, computation and communication efficiency \cite{ma2022state}. However, complex environments in edge IoTs lead to data heterogeneity as a widely acknowledged issue of 
distributed data collected in edge networks~\cite{fan2023efficient}. Non-independent and identically distributed (non-i.i.d.) datasets
 hinder distributed learning, leading to performance degradation, poor generalization, and inefficient cooperation~\cite{fan2024ganfed}.

This paper focuses on the challenges of data heterogeneity  in multi-access edge computing and addresses 
two key questions: \textit{how to effectively quantify  data heterogeneity?} and \textit{how to measure and mitigate 
its impact  on the performance of distributed learning?} 
%
This paper focuses on the widely used definition of non-i.i.d. data in terms of 
label distribution skew, which represents different joint probability of data and labels from distributed local datasets~\cite{kairouz2021advances}. 
Moreover, 
existing approaches for quantification of data heterogeneity often rely on artificially generated non-i.i.d. datasets using sampling methods such as the Dirichlet distribution~\cite{hsu2019measuring}, where a concentration parameter $\alpha$ controls the degree of heterogeneity. 
However, such a synthetic metric fails to capture the true statistical heterogeneity of real edge IoT datasets. 
Therefore, existing works on the impact of non-i.i.d. datasets on distributed learning still lack  effective quantification metric and theoretical analysis, which hinders their deployment in edge IoT of 6G networks.
Another challenge in multi-access distributed learning at the edge is the optimization of network intelligence~\cite{ma2022state}. That is, IoT deployments depend on careful network planning strategies to manage complex and costly device interactions. However, data heterogeneity further amplified  due to the diversity of edge devices creates barriers to collaborative network intelligence for multi-access edge computing. Existing works on FL with non-i.i.d. data have improved model optimization and aggregation via extracting common patterns from heterogeneous datasets~\cite{li2020federated}. However, these stochastic gradient descent (SGD)-based approaches struggle to correct the biased model updates caused by data heterogeneity, often leading to a failure of global model convergence~\cite{wang2024distributed,fan2023cb}. Hence,  deployment of multi-access distributed learning in edge networks still requires advanced network optimization  to ensure robustness and scalability in massive IoT devices.


Consider the data heterogeneity and network optimization issues at the edge, distributed swarm learning (DSL) is proposed to alleviate performance degradation caused by non-i.i.d. data and to mitigate the biased model aggregation\cite{wang2024distributed, fan2023cb}, by integrating particle swarm optimization with gradient descent. 
Despite these advantages, the vanilla DSL framework still has limitations. While the existing work of  DSL considers the non-i.i.d. issues, it has not analyzed how such datasets affect the collaborative learning process.  Moreover, its single best-worker selection mechanism is ineffective, as the collaboration gain of multi-user systems is not fully utilized. 
Facing 
these challenges, this paper proposes a multi-worker selection based 
DSL tailored for 
data heterogeneity scenarios. Our main contributions are summarized as follows: 
\begin{enumerate}
\item For efficient and robust distributed swarm learning with heterogeneous data, we introduce a new cooperation framework for DSL termed multi-worker selection based distributed swarm learning (M-DSL). To the best of our knowledge, we are the first to simultaneously consider the data heterogeneity and learning performance in evaluating individual contribution 
during worker selection, 
leading to a reliable M-DSL 
for edge IoT networks. 

\item We define 
a novel normalized non-i.i.d. degree metric to measure the data label distribution skew. 
The efficacy of the proposed metric is validated by observing the same trend of non-i.i.d. degree and the performance degeneration of distributed learning with non-i.i.d. datasets.

\item 
Based on the non-i.i.d. degree metric, we further develop a worker selection strategy which not only considers data heterogeneity but also enhances the collaborative training performance in global model aggregation. 
We also provide theoretical analysis on the convergence behavior guarantees of our proposed M-DSL solutions. 
\end{enumerate}

\section{Non-i.i.d. Formulation}\label{sec:noniid}
To 
analyze  the inherent characteristics of heterogeneous datasets and   evaluate their effect on performance degeneration in edge IoT networks, we  aim to develop an effective metric to quantify the non-i.i.d. impact on distributed learning. 
To this end, in this section, we  propose a general metric, termed the non-i.i.d. degree, to measure the  label distribution skew. 
Label distribution skew, also known as prior probability shift, i.e., the marginal distributions may vary across workers ${\Pr _{{D_i}}}\left( y \right) \ne {\Pr _{{D_j}}}\left( y \right)$ even if the same conditional distribution ${\Pr _{{D_i}}}\left( {y\left| x \right.} \right) = {\Pr _{{D_j}}}\left( {y\left| x \right.} \right)$~\cite{kairouz2021advances}.  
High skew  undermines the effective collaboration among workers in distributed learning, as various datasets may contain totally different labels. Inspired by the study on the non-overlapping distributions~\cite{lv2025wasserstein}, we develop a normalized metric based on the Wasserstein distance (WD) 
to evaluate the difference of the label distributions. 

\begin{figure}[tb]
    \setlength{\belowcaptionskip}{-0.5cm} 
  \centering
  \includegraphics[scale=0.18]{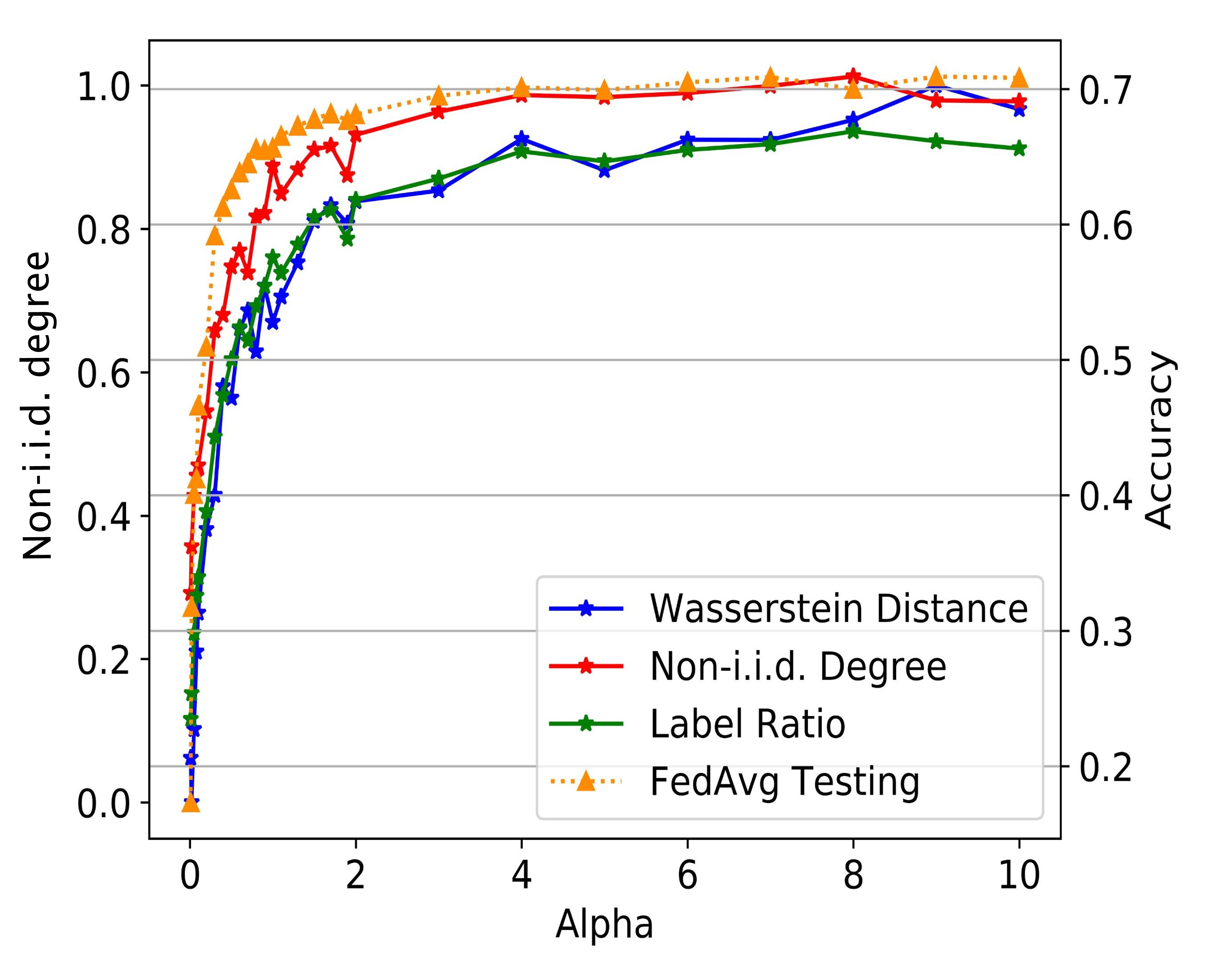}
    \caption{Experiments on data heterogeneity quantification. 
    We visualize the accuracy changes of FedAvg-based FL on heterogeneous data with different metrics: non-i.i.d. degree $1-\eta$, WD $1-W$, and label ratio $1-{{\left| {{{\cal L}_i}} \right|}}/{{\left| {{{\cal L}_g}} \right|}}$. 
    The trend of non-i.i.d. degree matches well with that of distributed learning accuracy as $\alpha$ varies, while obvious gaps are observed for the other two metrics i.e., WD and label ratio. }
    \label{fig:non-iid-degree}
\end{figure}

\setlength{\columnsep}{0.2in}

WD, first introduced to measure the model divergence in~\cite{zhao2018federated}, represents the minimum cost required to  transport one distribution into another. 
Given its capability to capture the geometric relationship and identify non-overlapping distributions, we are the first to apply WD for evaluating label distribution skew.  
Consider a learning based $L$-class classification task with distributed datasets at $C$ workers $\left\{ {{D_i}\left| {i = 1, \ldots ,C} \right.} \right\}$ and a global dataset ${{D_g}}$. 
For each dataset, its sample-label pairs  belong to a sample space $\mathcal{X}$ and a label space $\mathcal{L}$, as $D\in \{\mathcal{X}; \mathcal{L}\},\mathcal{L} = \{l_1, \ldots ,l_L\}$. The  function $f: \mathcal{X} \to \mathcal{L}$ represents the ground-truth mapping from the overall samples to their labels. 
The individual function $f_i$ describes the conditional distribution of labels $\mathcal{L}_i$ based on the samples $\mathcal{X}_i$ corresponding to the $i$-th dataset. 
However, data heterogeneity leads $f_i$ as a biased approximation to the ground-truth $f$. 
Considering that $D_g$ depicts the statistical characteristics of the global data, 
the WD of each dataset $D_i$ with respect to $D_g$ is quantified by: 

    \begin{equation}
     {W_i} = \mathop {\inf }\limits_{\gamma  \in \mathcal {F}\left( {{f_i},{f_g}} \right)} \int\limits_{{\cal X} \times {\cal X}} {\left\| {x - y} \right\|d\gamma \left( {x,y} \right)},    
     \label{eq:WD}  
    \end{equation}
where $\mathcal {F}\left( {{f_i},{f_g}} \right)$ denotes the joint distribution space such that each selected $\gamma$ follows the marginal distributions $f_i$ and $f_g$. 

However, we observe that WD fails to depict the consistency between the accuracy degradation of distributed learning model and the imbalance of  labels across workers, as indicated by the gap between the curve of WD and that of FedAvg testing  in Fig. \ref{fig:non-iid-degree}. 
This is because varying the concentration parameter $\alpha$~\cite{hsu2019measuring} for generation of different heterogeneous datasets,   WD focuses more on the difference of label distributions rather than the balance of label types.  
Inspired by the work
~\cite{zawad2021curse}, we introduce a ratio between the amount of  label types at worker-$i$ and that of global dataset, to describe 
the label diversity, leading to the normalized non-i.i.d. degree as  
 \begin{flalign}
    && {\eta _i} &= \text{Normalize}\left({\beta _1}\frac{{\left| {{{\cal L}_i}} \right|}}{{\left| {{{\cal L}_g}} \right|}} + {\beta _2}W_i + \varphi\right),   & \label{eq:noniid_degree} 
\end{flalign}
where $\operatorname{Normalize}\left(\cdot\right)$ represents the Min-Max Scaling~\cite{ali2014data},  
the label ratio $\frac{{\left| {{{\cal L}_i}} \right|}}{{\left| {{{\cal L}_g}} \right|}}$ represents the label difference between individual dataset $D_i$ and global dataset $D_g$ 
, $\left| {{{\cal L}_i}} \right|,\left| {{{\cal L}_g}} \right|$ are the sizes of individual and global label spaces. 
The hyperparameters ${\beta _1},{\beta _2}, \varphi $ represent the significance of label diversity, that of statistical distribution difference, and that of fitting residual, which are determined by the dataset generation and experiment settings. 
The significantly diminished gap between the curve of our proposed non-i.i.d. degree metric $\eta$ and that of FedAvg testing reflected in Fig.~\ref{fig:non-iid-degree} illustrates its effectiveness in assessment for both data heterogeneity measurement and the non-i.i.d. impact on learning performance. 

\section{Multi-worker Selection based Distributed Swarm Learning}

This section investigates  worker selection and collaboration strategy for DSL  considering data heterogeneity. We start with  problem formulation and then develop the M-DSL algorithm. 

\subsection{Problem Formulation} \label{Problem Formulation}

Distributed learning systems allow multiple workers to collaboratively train a general model given their individual local data. 
Suppose the DSL framework with the parameter server (PS) and $C$ workers, where each worker has its local dataset ${D_i} = \left\{ {\left( {{x_{i,k}},{l_{i,k}}} \right)} \right\},{x_{i,k}} \in {\cal X},{l_{i,k}} \in {\cal L}$. In DSL, workers also have a synthetic global dataset $D_g$ for function value evaluation\cite{wang2024distributed, fan2023cb}. In each communication round, workers train $ M\left( {\mathbf{w}_{i,t},D} \right)$ locally and get their evaluation score $F_{i,t}$, where the parameter vector $\mathbf{w}_{i,t}=\{w_{i,t}^1, \ldots,w_{i,t}^n\} \in {\cal R}^N$ contains the parameters of the trained model. Based on the real time assessment on function values $\{F_{i,t}\}$, in this work, we allow the PS to select multiple workers to upload their trained parameters $\{\mathbf{w}_{i,t}\}$ to update the global model $\{\mathbf{w}_{i,t}^g\}$, which is different from the vanilla DSL where only one worker is selected in \cite{fan2023cb}. The global model $ M\left( {\mathbf{w},D} \right)$ is iteratively aggregated over the  trained model $ M_{i,t}\left( {\mathbf{w}_{i,t},D_i} \right)$ from the selected workers to approximate the mapping $f: \mathcal{X} \to \mathcal{L}$. 

\subsection{Distributed Swarm Learning}

 Different from the FL, DSL~\cite{wang2024distributed} is a recently developed distributed learning system incorporating the particle swarm optimization (PSO) algorithm with gradient descent techniques to alleviate accuracy reduction caused by data heterogeneity. DSL establishes a distributed network between the PS and devices (so-called workers) for cooperative learning. Details of the vanilla DSL algorithm can be found in~\cite{fan2023cb}.  For PSO evaluation, DSL introduces a global dataset $D_g$ to assess individual model training performance on-the-fly. The loss function is defined as the Root Mean Square Error: 
\begin{flalign}
    && {F_i}\left( {{\mathbf{w}_{i,t+1}};{D_i}} \right) &= \frac{1}{{\left| {{D_i}} \right|}}\sum\limits_{({x_{i,j}},{l_{i,j}}) \in {D_i}} \!\!\!\!\!{\sqrt {{{\left( {M({\mathbf{w}_{i,t}},{x_{i,j}}) \!- {l_{i,j}}} \right)}^2}} }. & \label{eq:RMSE}  
    \end{flalign}
For simplicity of symbols, we use $F_{i,t + 1}$ to represent ${F_i}\left( {{\mathbf{w}_{i,t+1}};{D_i}} \right)$ hereafter.


\subsection{Multi-worker Selection Mechanism}
Multi-worker selection is inspired to fully utilize the collaboration gain of network intelligence techniques by encouraging the effective and efficient 
participation from multiple workers. Based on the non-i.i.d. degree proposed in Section~\ref{sec:noniid}, the multi-worker selection mechanism is designed to train the general model with heterogeneous data, when considering both the data heterogeneity and model performance. 

The worker selection indicator  is $S_t=\{s_{1,t},\ldots, s_{C,t}\},s_i \in \{0,1\}$, in which $s_{i,t}=1$ represents the $i$-th worker selected in the $t$-th communication round or $s_{i,t}=0$ otherwise. Accordingly, the objective  is to maximize workers participation:   
\begin{flalign}
    &&   J\left( {{S_t}} \right) &= \max{\sum\limits_{i = 1}^C {{s_{i,t}}} }.   & \label{eq:objective} 
\end{flalign}

While the vanilla DSL only considers the value of loss function \eqref{eq:RMSE} to evaluate the model for single worker selection, we further introduce the non-i.i.d. degree to the loss function to better learn the underlying features from distributed datasets via effective collaborative learning among multiple workers in the presence of non-i.i.d. datasets. For the local dataset $D_i$ and the loss function $F_{i,t}$, we utilize a trade-off between individual learning performance and data quality as 
\begin{flalign}
    &&  \theta_{i,t} &=  {{\tau}F_{i.t} + {\left(1 - \tau\right)}\eta _{i}} ,&\label{eq:weighted_constrain}
\end{flalign}
where the regularizer ${\tau}$ is applied to balance the learning model performance via $F_{i.t}$ and the importance of the data quality by measuring  data non-i.i.d. degree $\eta _{i}$  defined in \eqref{eq:noniid_degree}. 

Based on this new trade-off score, we formulate the  multi-worker selection scheme as 
\begin{flalign}
    &&  \theta_{i,t}s_{i,t} &\le \overline \theta  _{t-1} ,i = 1, \ldots ,C, &\label{eq:acc_constrain}
\end{flalign}
where ${{\bar \theta }_{t - 1}} = \frac{1}{C}\sum\limits_{i \in C} {{\theta _{i,t - 1}}}$ reveals the adaptive selection process based on the averaged gain over all workers in the previous communication round. 

Then, the global model parameter updating  is expressed as
\begin{flalign}
    &&  {\mathbf{w}_{t + 1}} &= \mathbf{w}_{t} + \frac{1}{{\sum\limits_{i = 1}^C {{s_{i,t + 1}}} }}\sum\limits_{i = 1}^C {{s_{i,t + 1}}\left({\mathbf{w}_{i,t + 1}} - {\mathbf{w}_{i,t}}\right)}. &\label{eq:multi_updates}
\end{flalign}

Since only the selected workers are invited to contribute to global model updating via local model uploading, multi-worker selection costs less communication energy than all-worker uploading for model updates. 
Then, for local model updates, we design the individual model parameter updating process for M-DSL as 
\begin{flalign}
    &&  {\mathbf{w}_{i,t + 1}} \!=& \mathbf{w}_{i,t} + {c_0}{v_{i,t}} + {c_1}\left\| {\mathbf{w}_{i.t}^l - {\mathbf{w}_{i,t}}} \right\| + {c_2}\left\| {\mathbf{w}_t^{\overline{ g} } - {\mathbf{w}_{i,t}}} \right\| & \nonumber \\
    && &- \alpha \nabla F\left( {{\mathbf{w}_{i,t}}, D_g} \right), & \label{eq:M-DSL_update}  
\end{flalign}
where the local and  global best parameters are updated as
\begin{flalign}
    && \mathbf{w}_{i,t}^l &=\mathop {\arg \min }\limits_{w\in\{w_{i,t-1},w_{i,t}\}} \left\{ {F_{i,t-1},F_{i,t}} \right\} &\nonumber\\
    && & = \mathbb{I}_{\left( {F_{i,t} > F_{i,t-1}} \right)}\mathbf{w}_{i,t-1} + \mathbb{I}_{\left( {F_{i,t-1} > F_{i,t}}  \right)}\mathbf{w}_{i,t}. & \label{eq:Local_update} \\
    &&  \mathbf{w}_{i,t}^{\overline{g}} &=  \mathop {\arg \min }\limits_{\mathbf{w}\in\{\mathbf{w}_{t-1},\mathbf{w}_{t}\}} \left\{{F_{t-1},F_{t}} \right\} & \nonumber \\
    && &=\mathbb{I}_{\left( {F_{t} > F_{t-1}} \right)}\mathbf{w}_{t-1} + \mathbb{I}_{\left( {F_{t-1} > F_{t}}  \right)}\mathbf{w}_{t} & \label{eq:Global_update}   
\end{flalign}
where the event indicator is designed as $\mathbb{I}_{F_{t} > F_{t-1}}$, whose value is equal to $1$ when $F_{t} > F_{t-1}$, or $0$ otherwise. 

\subsection{Algorithm Implementation}
The implementation of our M-DSL is described in Algorithm.~\ref{alg:MDSL}.  For initialization, all the workers are invited in the first  round, with randomized $\mathbf{w}_{i,t}^l,\mathbf{w}_{i,t}^{\overline{g}}$. 
Then, in each round, each worker trains local model based on the updated $\mathbf{w}_{i,t + 1}^l, \mathbf{w}_{i,t + 1}^{\overline{g}}$. Next, M-DSL selects workers per \eqref{eq:acc_constrain} to upload their models. The contributions of selected workers are aggregated for global model updates, until model converges.

\begin{algorithm}[!htb]
	\caption{M-DSL}
	\label{alg:MDSL}
	\begin{algorithmic}[1]
\renewcommand{\algorithmicrequire}{\textbf{Initialization:}}
		\REQUIRE ~~\\
		Initialize $\mathbf{w}, \mathbf{w}_{i,t}^l,\mathbf{w}_{i,t}^{\overline{g}}$, $t,i$, given datasets $\left\{ {{D_i}} \right\}, D_g$;
\\
        Calculate the non-i.i.d. degree $\{\eta_i\}$;
    \FOR {each communication round $t=1:T$}

        \STATE \!\!\!\textbf{ at each local worker:$i \in C$}
            \STATE \hspace{0.1in} updates $\mathbf{w}_{i,t}^l,\mathbf{w}_{i,t}^{\overline{g}}$ via \eqref{eq:Local_update} and \eqref{eq:Global_update}; 
            \STATE  \hspace{0.1in} update $w_{i,t + 1}$ via \eqref{eq:M-DSL_update}, calculate $F_{i,t + 1}$ via \eqref{eq:RMSE}; 
            \STATE   \hspace{0.1in} search for better workers via \eqref{eq:objective} and \eqref{eq:acc_constrain}; 
            \STATE   \hspace{0.1in} upload $\left\{ {{w_{i,t + 1}},{\theta_{i,t + 1}}} \right\},i \in {S_{{{t + 1}}}}$ to PS;

        \STATE \!\!\!\textbf{ at the PS:} 
        \STATE \hspace{0.1in} update global model via \eqref{eq:multi_updates};
        \STATE \hspace{0.1in} {broadcast $w_{t + 1}$ to all workers.}
    \ENDFOR
	\end{algorithmic}
\end{algorithm}


\section{Performance Analysis}
In this section, we analyze   convergence behavior and communication efficiency of M-DSL.
\subsection{Assumption}
Following a similar approach in \cite{fan2021joint}, we adopt the common assumptions as follows.
\begin{assumption}\label{ass1}
($\alpha$-order Lipschitz continuity, smoothness): 
{\color{black}In $t$-th communication round, the gradient $\nabla F_i(\mathbf{w}_{i,t})$ of the loss function $F_i(\mathbf{w})$ at worker $i$} is $\alpha$-order uniformly Lipschitz smooth with respect to $w$, that is,
\begin{eqnarray}\label{eq:Lipschitz}
\|\nabla F_i(\mathbf{w}_{i,t+1})-\nabla F_i(\mathbf{w}_{i,t})\|_{\alpha}\leq L\|\mathbf{w}_{i,t+1}-\mathbf{w}_{i,t}\|_{\alpha}, \;\;\;  
\end{eqnarray}
where $L$ is a positive constant, referred as the Lipschitz constant for $F(\cdot)$~\cite{fan2023cb,fan2021joint,xu2024qc}.
\end{assumption}

\begin{assumption}\label{ass2}
(Bounded local gradients): 
\color{black} For the DSL framework, the gap between the global model $w_t$ and the local model $w_{i,t},\forall i,t$ is bounded by ~\cite{fan2021joint, fan2023cb, fan20221}
\begin{eqnarray}\label{eq:bound_local_gradients}
{{{w_{i,t}}-{w_{t - 1}}}\le \max\limits_{i\in\{1, \ldots,C\}}  \left\{ {{\alpha \nabla F_{i.t}+ {\left( {{c_0} - {c_1} - {c_2}} \right){v_{i,t}}} }} \right\}}.
\end{eqnarray}

\end{assumption}

\subsection{Convergence Bound}
With the \textbf{Assumption~\ref{ass1}} and \textbf{Assumption~\ref{ass2}}, the convergence errors of M-DSL are bounded by \textbf{Theorem~\ref{theorem1}}

\begin{theorem}\label{theorem1}
The individual parameter updates function is given by \eqref{eq:M-DSL_update}. For the global optimal model ${w^{\star}}$ and communication round $t=1,\ldots,T$, given the learning rate $\alpha=\frac{1}{L}$, the expected convergence rate is bounded by
\begin{align}\label{eq:theorem1}
\mathbb{E}\left[\sum_{t=1}^{T}\frac{\|\nabla F_i(\mathbf{w}_{i,t})\|^2}{T}\right]&\le \frac{F\left(\mathbf{w}_{i,0}\right)-F\left(\mathbf{w}^{\star}\right)}{\mathbb{E}\left[{\overline{\Phi}}\right] T },\ \forall i, 
\end{align}
where $\overline{\Phi}= k_1\overline{u}_i\overline{q}_i+Lk_1\overline{u}_i^2+\frac{\left(1+c_2\right)^2}{2L} + k_2\overline{u}\overline{q}+Lk_2^2\overline{u}^2$ represents the upper bound indicating all the workers are selected. The coefficient of individual velocity is $k_1=c_0 - \mathbb{I}_{\left( {F_{i,t} > F_{i,t-1}} \right)}c_1 - c_2\left(c_0-c_1-c_2\right)$. The coefficient of global velocity is $k_2=c_2\mathbb{I}_{\left( {F_{t-1} > F_{t}}  \right)}$.
\end{theorem}
\begin{proof}
A sketch of proof is provided here for the page limit.

First, the inequality of the loss function and its gradient can be derived according to the \textbf{Assumption~\ref{ass1}}.
\begin{align}
F_{i,t+1} \le & F_{i,t} + {\mathbf{v}_{i,t+1}^{\rm{T}}}\nabla F_{i,t}+ \frac{L}{2}{\left\| \mathbf{v}_{i,t+1} \right\|^2}.\label{eq:smooth}
\end{align}
The velocity update of individual model $\mathbf{v}_{i,t+1}=\mathbf{w}_{i,t+1}-\mathbf{w}_{i,t}$ can be expressed by substituting \eqref{eq:M-DSL_update}, \eqref{eq:Global_update} and \eqref{eq:Local_update}:
\begin{align}\label{eq:velocity_1}
&{\mathbf{v}_{i,t + 1}} = c_0\mathbf{v}_{i,t} - \alpha \nabla F_{i,t} \nonumber\\
&+ c_2\left(\mathbb{I}_{\left( {F_{t} > F_{t-1}} \right)}\mathbf{w}_{t-1} + \mathbb{I}_{\left( {F_{t-1} > F_{t}}  \right)}\mathbf{w}_{t} - \mathbf{w}_{i,t}\right) 
\nonumber\\
& + c_1\left(\mathbb{I}_{\left( {F_{i,t} > F_{i,t-1}} \right)}\mathbf{w}_{i,t-1} + \mathbb{I}_{\left( {F_{i,t-1} > F_{i,t}}  \right)}\mathbf{w}_{i,t} - \mathbf{w}_{i,t}\right)  \nonumber\\
&\le \left(c_0 - \mathbb{I}_{\left( {F_{i,t} > F_{i,t-1}} \right)}c_1 - c_2\left(c_0-c_1-c_2\right) \right)\mathbf{v}_{i,t}\nonumber\\
&-\left(1+c_2\right)\alpha \nabla F_{i,t} + c_2\mathbb{I}_{\left( {F_{t-1} > F_{t}}  \right)}\mathbf{v}_t\nonumber\\
&=k_1\mathbf{v}_{i,t}+k_2\mathbf{v}_{t}-\left(1+c_2\right)\alpha \nabla F_{i,t},
\end{align}
where $\mathbf{v}_t=\mathbf{w}_t-\mathbf{w}_{t-1}=\frac{1}{{\sum\limits_{j = 1}^C {{s_{j,t}}} }}\sum\limits_{j = 1}^C {{s_{j,t}}{\mathbf{v}_{j,t}}}$ represents the velocity update of global model, which is related to the multi-worker selection strategy. As 
in \cite{fan2023cb}, we define the bounded cosine-similarity of $\mathbf{v}_{i,t}$ and $ -\nabla F_{i,t}$, $v_{t}$ and $ -\nabla F_{i,t}$ as
\begin{align}
&\underline{q}_i\le \frac{{\color{black}\langle\mathbf{v}_{i,t}, -\nabla F_{i,t}\rangle}}{\|\mathbf{v}_{i,t}\|\|\nabla F_{i,t}\|} \le \overline{q}_i, \ \forall i, t,\label{eq:bound_1} \\
&\underline{q}\le \frac{{\color{black}\langle\mathbf{v}_{t}, -\nabla F_{i,t}\rangle}}{\|\mathbf{v}_{t}\|\|\nabla F_{i,t}\|}\le \overline{q}, \ \forall i, t,\label{eq:bound_2} \\
&\underline{u}_i \le \frac{\|\mathbf{v}_{i,t}\|}{\|\nabla F_{i,t}\|}\leq\overline{u}_i, 
\underline{u} \le \frac{\|\mathbf{v}_{i,t}\|}{\|\nabla F_{i,t}\|}\leq\overline{u}_i, \ \forall i, t.\label{eq:bound_3}
\end{align}
Then substituting \eqref{eq:velocity_1} and the aforementioned bounds to scale the \eqref{eq:smooth}, given $\alpha=\frac{1}{L}$, we have
\begin{align}
F_{i,t+1} -F_{i,t} \le -\Phi_{i,t} \left\|\nabla F_{i,t}\right\|^2, \label{eq:scaling_1}
\end{align}
where the $\Phi_{i,t}=\Phi+\Phi_{i,t}^s$, in which $\Phi = k_1\overline{u}_i\overline{q}_i+Lk_1\overline{u}_i^2+\frac{\left(1+c_2\right)^2}{2L}$ is time-invariant. $\Phi_{i,t}^s=\mathbb{I}_{\left( s_{i,t}=1\right)}\left(k_2\overline{u}\overline{q}+Lk_2^2\overline{u}^2\right)$ represents the impact of the worker selection strategy on local parameter updating. Noting that if the $i$-th worker is not selected, $\mathbf{v}_t$ is constant in the updating of $v_{i,t+1}$. Hence, $\mathbb{I}_{\left( {s_{i,t}=0} \right)}\mathbf{v}_t$ does not contribute to the update of individual velocity $\mathbf{v}_{i,t+1}$.

Finally we extend a sum of expected convergence error with the ideal upper bound $F\left(\mathbf{w}_{i,0}\right)-F\left(\mathbf{w}^{\star}\right)$ over $T$ iterations as
\begin{align}
F\left(\mathbf{w}_{i,0}\right)-F\left(\mathbf{w}^{\star}\right) &\ge \mathbb{E}\left[\sum\limits_{t = 0}^T{F_{i,ti1} -F_{i,t}}\right]\nonumber\\
& \ge \mathbb{E}\left[\overline{\Phi}\sum\limits_{t = 0}^T{ \left\|\nabla F_{i,t}\right\|^2}\right].\label{eq:scaling_2}
\end{align}

The convergence rate is limited by 
\begin{align}
&  \mathbb{E}\left[\frac{1}{T}\sum\limits_{t = 0}^T{ \left\|\nabla F_{i,t}\right\|^2}\right]\le \frac{F\left(\mathbf{w}_{i,0}\right)-F\left(\mathbf{w}^{\star}\right)}{\mathbb{E}\left[{\overline{\Phi}}\right] T }. \label{eq:scaling_3}
\end{align}
Here, the proof is completed.
\end{proof}
In this way, our proposed M-DSL algorithm has a similar convergence rate, a worst-case $ \mathcal{O}\left(\frac{1}{T} \right)$, as standard DSL~\cite{fan2023cb}. 

\subsection{Communication Efficiency}
As a biased worker selection based framework, partial model updates achieve comparable performance about vanilla FL with less message transmission in each communication round. Specifically, given $C$-participants distributed learning systems for training a model with $n$ length parameter vector, the communication overhead of FedAvg is $nC$ for uploading local updates. M-DSL accomplishes less information uploads as $n\sum\limits_{i = 1}^C {{s_{i,t}}}$. In the next section, experimental results show that a small subset of workers can represent the contributions of all participants after the early training stage, reducing the communication overhead of M-DSL in model updates. Moreover, M-DSL converges in fewer rounds than FedAvg, indicating its deployment requires less wireless resource.

\section{Simulation Results}
In this section, we provide simulation results to evaluate the performance of M-DSL in edge IoT networks. Compared with the benchmark DSL~\cite{fan2023cb} and FedAvg~\cite{mcmahan2016federated}, M-DSL achieves faster convergence and higher validation accuracy.

\begin{figure}[tb]
  \centering
  \includegraphics[scale=0.26]{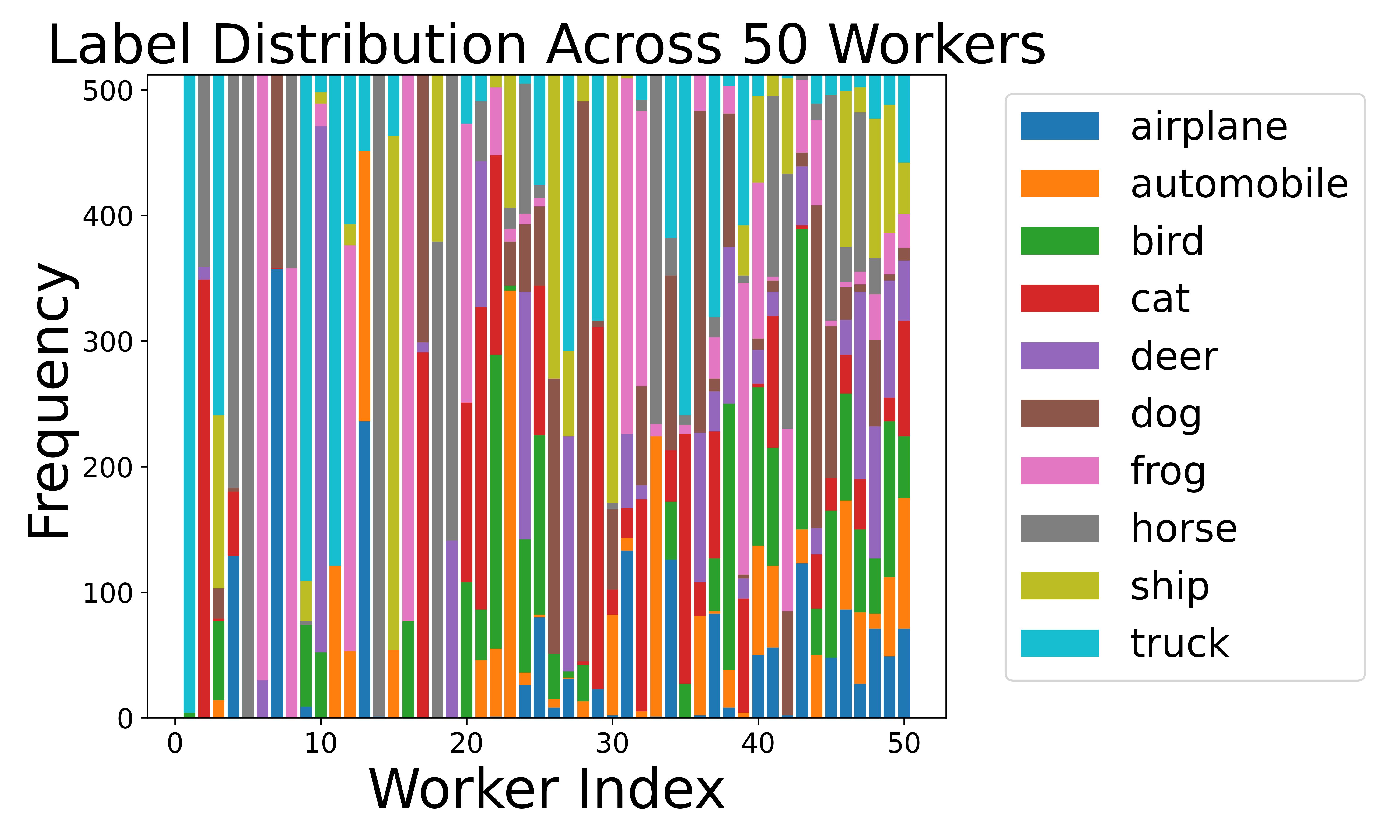}
    \caption{The non-i.i.d. CIFAR10 case II: A more common heterogeneous example in edge IoTs, the network contains 20 datasets sampled by $\alpha =0.1$, 15 datasets by $\alpha=0.5$, 10 datasets by $\alpha=1$, 5 datasets by $\alpha=10$.}\label{datasets}
\end{figure}

\subsection{Experimental Settings}
We test our M-DSL in an edge IoT network with $C=50$ local workers, where each worker contains the individual dataset ${\left| {{D_i}} \right|}=512$ and a synthetic evaluation dataset ${\left| {{D_g}} \right|}=2048$. Local datasets are generated via the time-invariant subset sampling from public datasets MNIST ~\cite{deng2012mnist} and CIFAR10 ~\cite{krizhevsky2009learning}. The data sampling process is controlled by the Dirichlet distribution with parameter $\alpha$. $D_g$ is generated by GANs. The neural network model of M-DSL is a five-layer Convolutional Neural Network~\cite{fan2023cb} and the ResNet18~\cite{casella2024experimenting}, respectively. 
For the hyperparameter settings, we apply the SGD optimizer with attenuated learning rate $\alpha_{init}=0.01,\gamma=0.5$. The training process involves $40$ communication rounds on CIFAR10 and $20$ round on MNIST, each round includes $4$ epochs. The batch size is $64$ for stable training. 
The hyperparameters used in PSO are sampled from $c_0 \sim U\left( {0,1} \right)$ and $c_1,c_2\sim \mathcal{N}\left( {0,1} \right)$. The model aggregation  is conducted via averaging the selected models with the weight regularizer $\tau=0.9$ in \eqref{eq:weighted_constrain}.

\begin{figure*}[!t]
    \centering
    \subfloat[I.i.d. case on MNIST]{\includegraphics[width=1.8in]{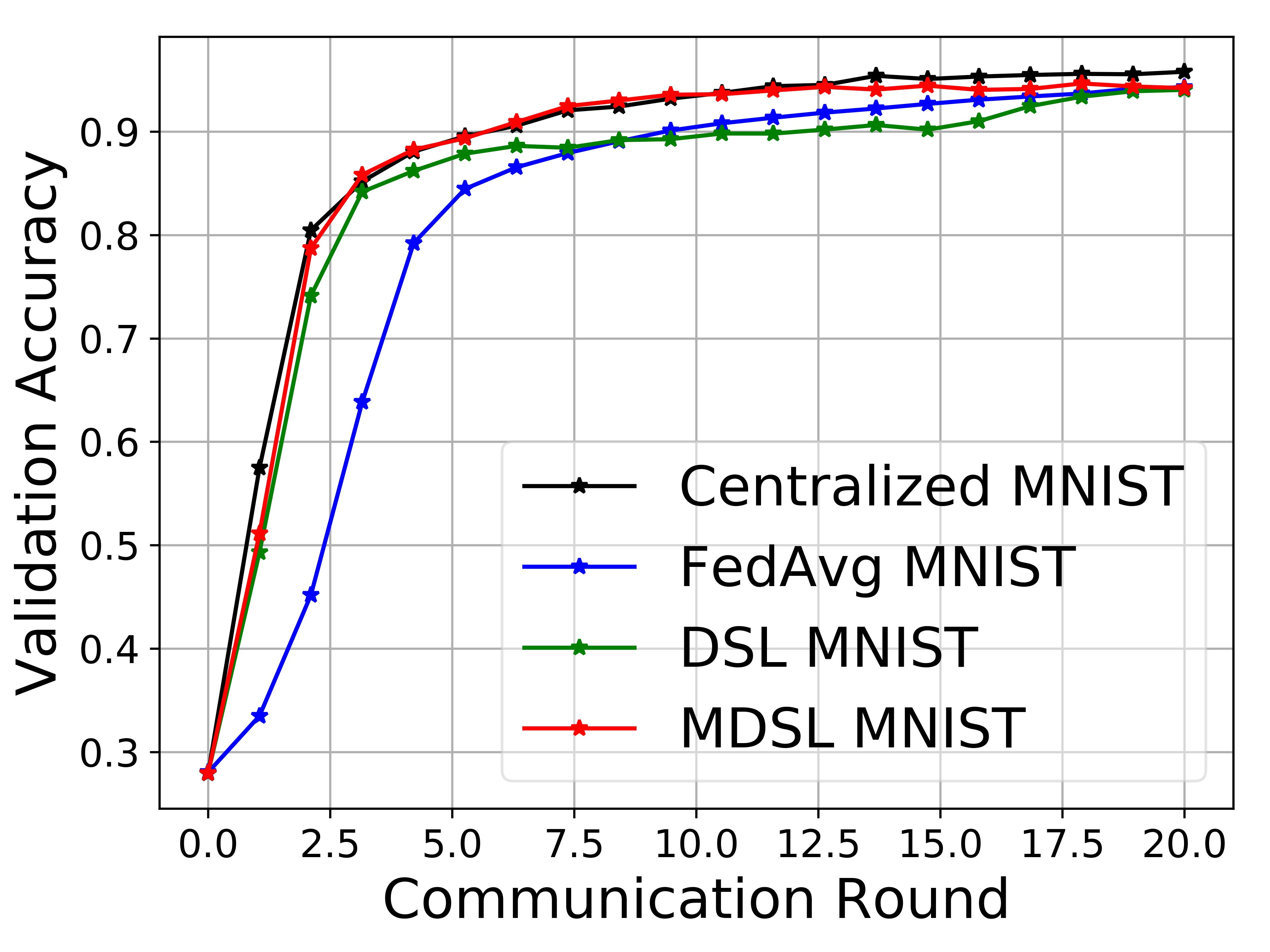}}%
    \label{results:figure1}
    \hfil
    \subfloat[Non.i.i.d. Case I on MNIST]{\includegraphics[width=1.8in]{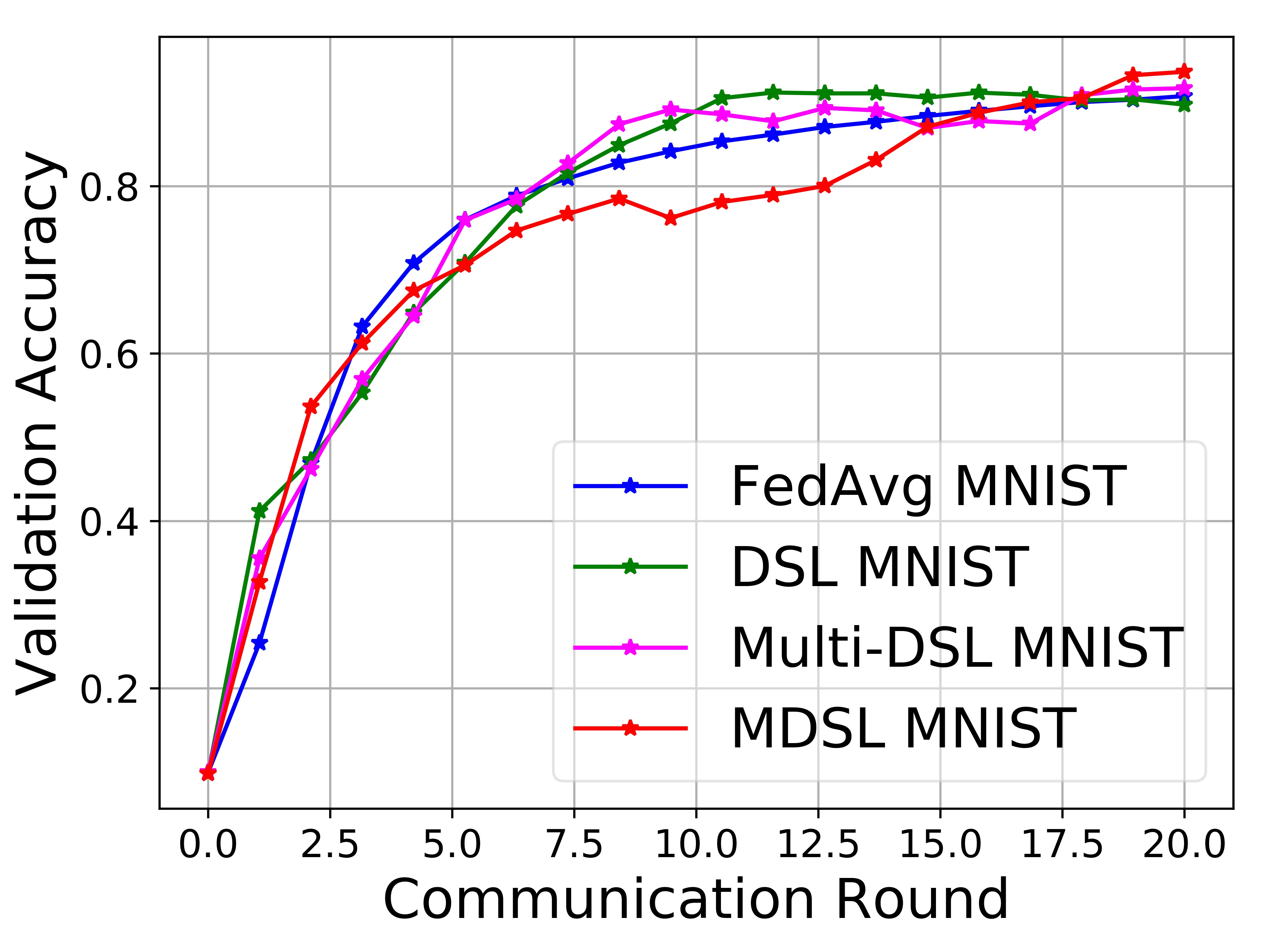}}%
    \label{results:figure2}
    \hfil
    \subfloat[Non.i.i.d. Case II on MNIST]{\includegraphics[width=1.8in]{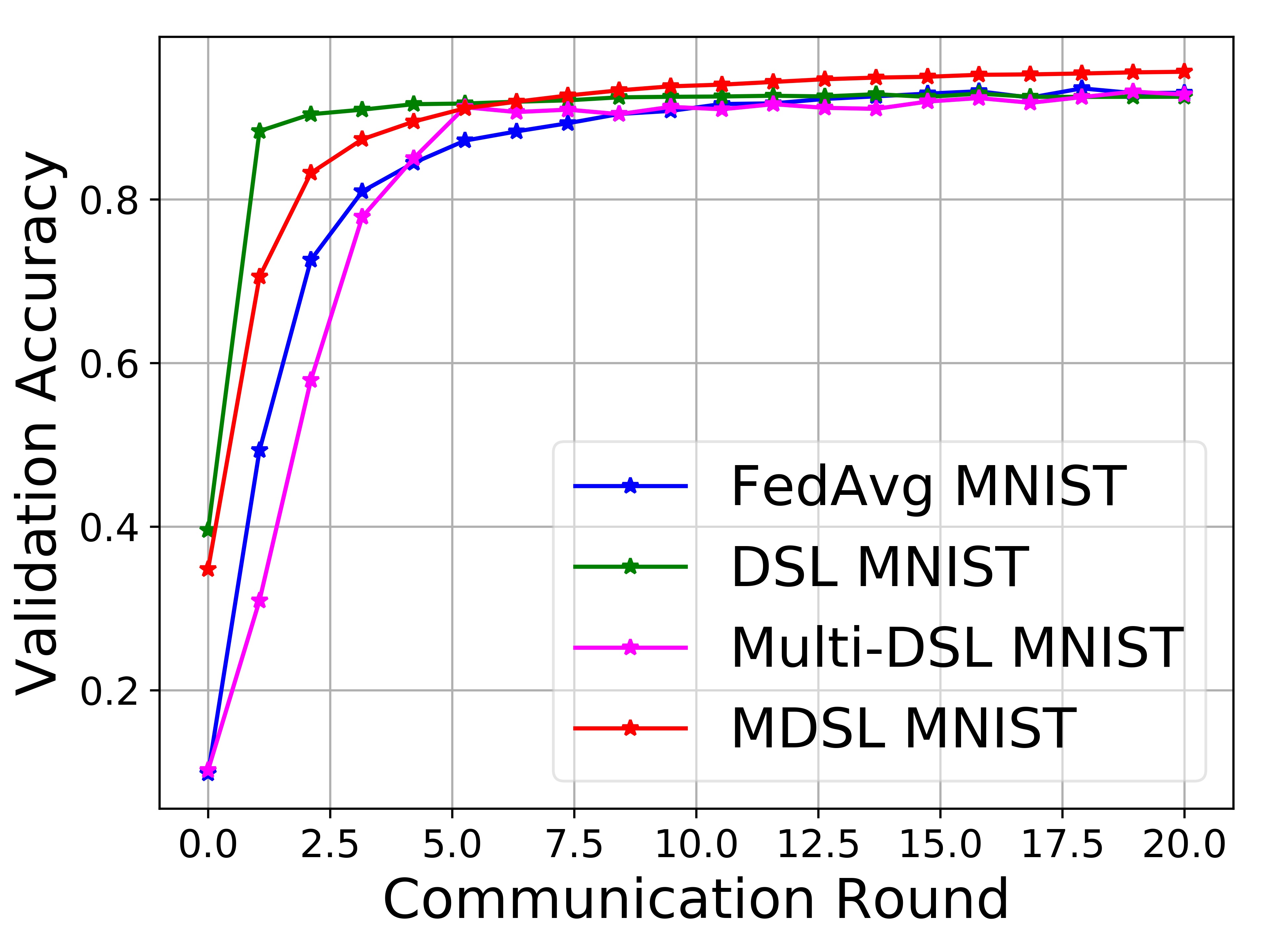}}%
    \label{results:figure3}\\
    \subfloat[I.i.d. Case on CIFAR10]{\includegraphics[width=1.8in]{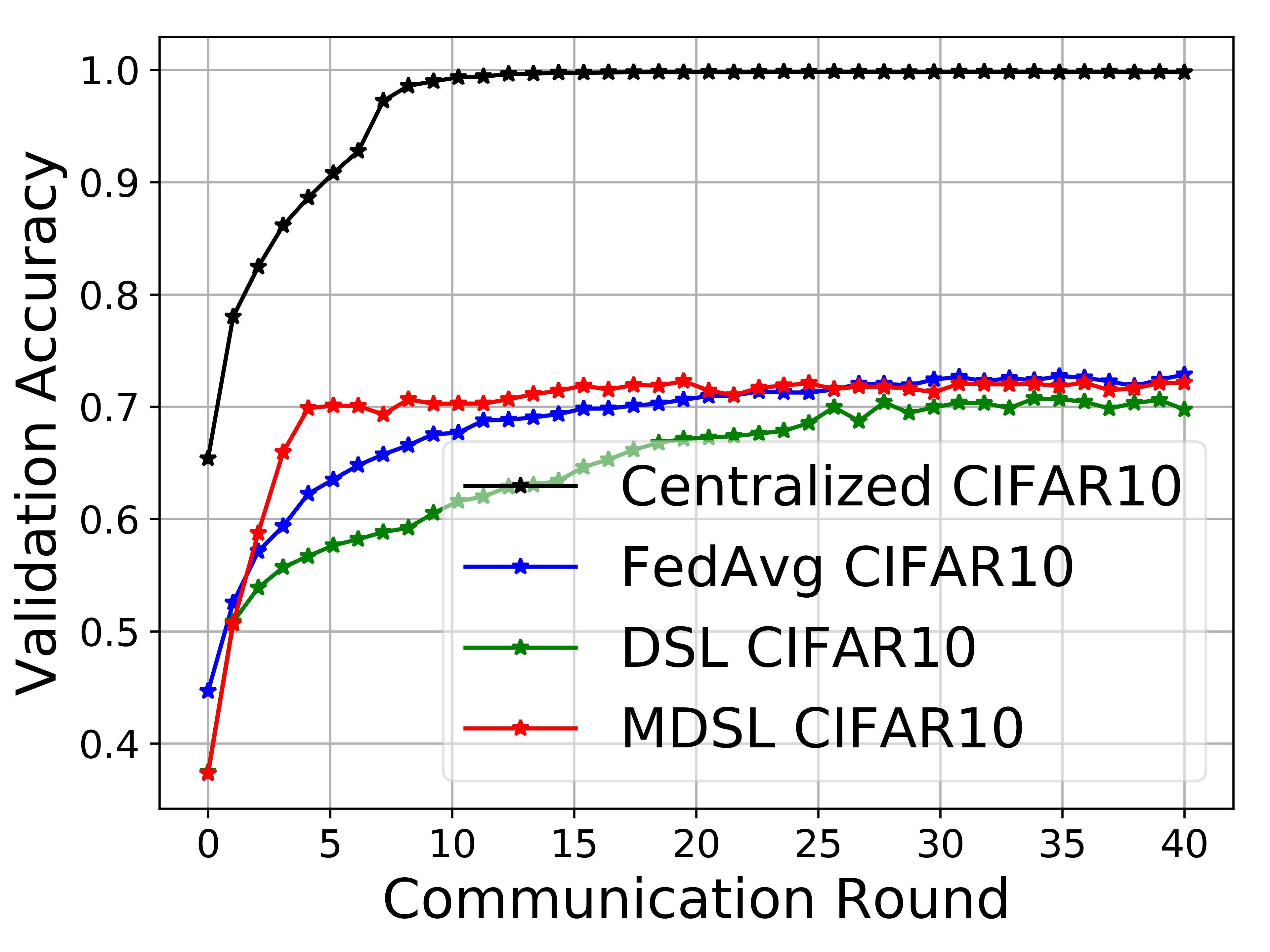}}%
    \label{results:figure4}
    \hfil
    \subfloat[Non.i.i.d. Case I on CIFAR10]{\includegraphics[width=1.8in]{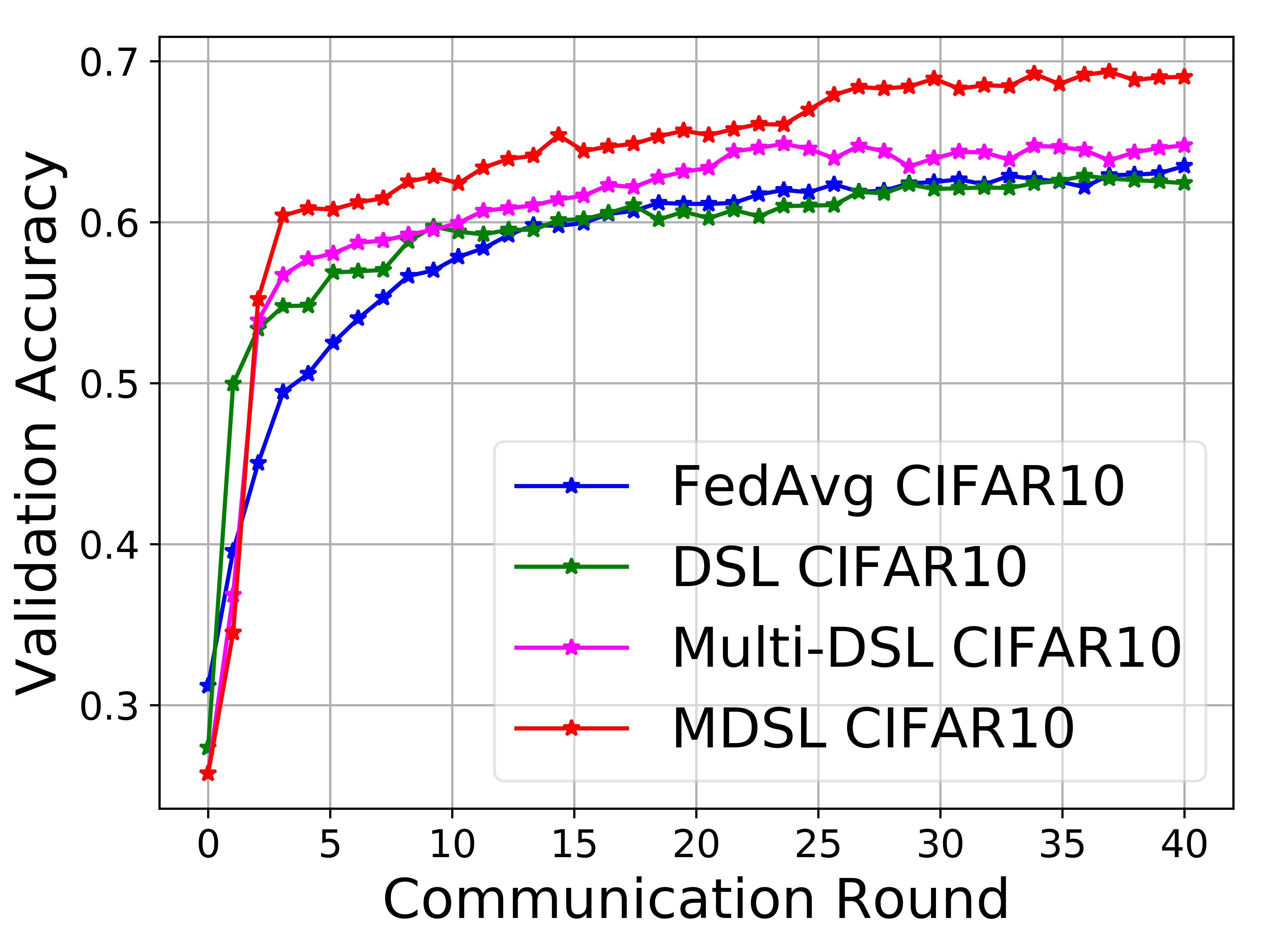}}%
    \label{results:figure5}
    \hfil
    \subfloat[Non.i.i.d. Case II on CIFAR10]{\includegraphics[width=1.8in]{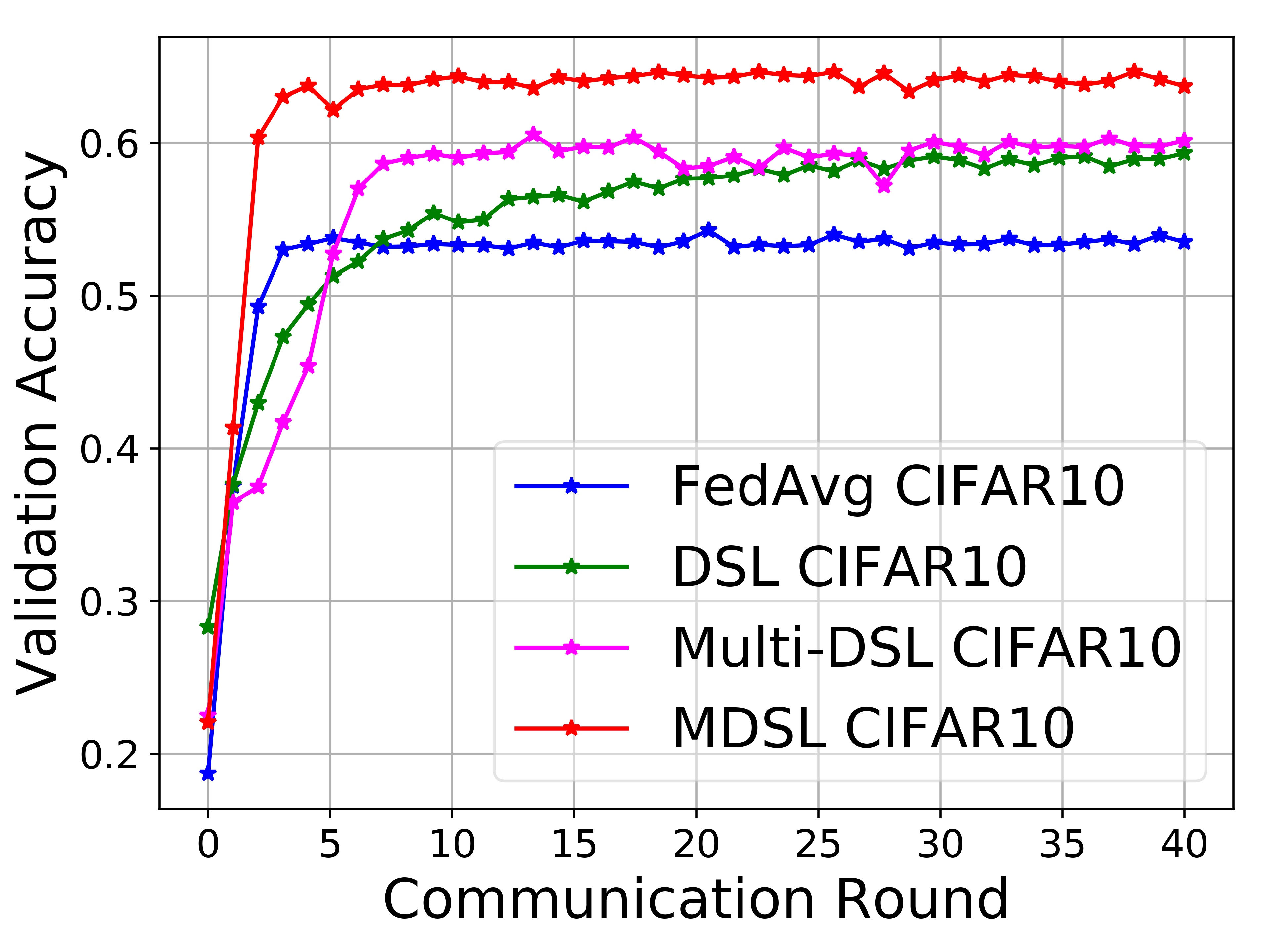}}%
    \label{results:figure6}
    \vspace{-0.03in}
    \caption{Learning performance evaluation of image classification by FedAvg, DSL and improved DSL. 
}
    \label{results}
    \vspace{-0.15in}
\end{figure*}
    

\subsection{Performance Comparison}
To evaluate different distributed learning methods, we design experiments to accomplish classification tasks, the multi-worker selection DSL without non-i.i.d. degree (Multi-DSL), and the multi-worker selection DSL (M-DSL) for different data cases, including i.i.d. case: 50 i.i.d datasets, non-i.i.d. case I: 50 non-i.i.d. datasets sampled by Dirichlet Distribution $\alpha=0.5$, and 
more diverse non-i.i.d. case II: 20 sets sampled by $\alpha=0.1$, 15 sets by $\alpha=0.5$, 10 sets by $\alpha=1$ and 5 sets by $\alpha=10$. The non-i.i.d. case II provides a more practical data heterogeneous scenario with diverse datasets as visualized in Fig. \ref{datasets}. Classification accuracy of M-DSL is compared with baseline approaches DSL and FedAvg in Fig. \ref{results}. The i.i.d. case plays as a baseline performance that M-DSL can reach in the ideal case. 
Experiments in non-i.i.d. cases I and II highlight the advancement of M-DSL to achieve more accurate accuracy than other benchmarks 
on heterogeneous datasets. Moreover, M-DSL achieves more efficient convergence with higher validation accuracy in the non-i.i.d. cases. The introduction of multi-worker selection obviously enhances the convergence rate and classification accuracy by the full utilization of distributed data. The non-i.i.d. degree metric introduced in this work further strengthens the enhancement of our M-DSL in  heterogeneous scenarios.

\subsection{Hyperparameter Design}\label{subsec:hyper}
This section discusses the hyperparameter design of the non-i.i.d. degree in \eqref{eq:noniid_degree}. For the optimum design of $\beta_1,\beta_2, \varphi $, we observe the distribution difference $W$, label ratio ${{\left| {{{\cal L}_i}} \right|}}/{{\left| {{{\cal L}_g}} \right|}}$ and FedAvg accuracy $acc$ on CIFAR10 with the non-i.i.d setting $\alpha = {0.001, \ldots, 1000}$, and attempt to fit the trend of $\eta$ with $acc$. Applying the least squares based fitting of the observations, we use $90\%$ records to fit the linear function and $10\%$ to test. The regression result shows the remarkable linearity with $R^2=0.97$ in MNIST and $R^2=0.895$ in CIFAR10. The results of linear regression guide a general setting of $\beta_1=0.286,\beta_2=-0.07,\varphi =0.592$ in CIFAR10 and $\beta_1=-0.031,\beta_2=0.127,\varphi =-0.04$ in MNIST, in which the non-i.i.d. degree can take place of the generation parameter $\alpha$. Thus, the proposed non-i.i.d degree not only plays as an effective metric to describe the  heterogeneity level of distributed data in edge IoT devices, but also contributes to guide the multi-worker selection in DSL. 


\section{Conclusion}
This work proposes a M-DSL algorithm to tackle the data heterogeneity challenge in edge IoT networks. 
We define a non-i.i.d. degree metric superior in measuring data heterogeneity of distributed datasets and accuracy degradation due to the non-i.i.d. impact on distributed learning. 
Further, we  design  a multi-worker selection strategy that enhances  network intelligence under the constraint of limited network resources. In addition, we analyze the convergence behavior of the proposed M-DSL, supported by comprehensive simulations, which indicate the advantages of M-DSL in terms of faster convergence, higher accuracy, and less communication cost than FedAvg and standard DSL with heterogeneous data. 

\bibliographystyle{IEEEtran}
\bibliography{bib}

\end{document}